\documentclass[a4paper]{article}
\usepackage[utf8]{inputenc}
\usepackage{tikz}
\usepackage{amsmath}
\usepackage{amssymb}
\usepackage{amsthm}
\usepackage{textcomp}
\usepackage{dsfont}
\usepackage{wrapfig}
\usepackage{a4wide}
\usepackage{authblk}

\usepackage{amsfonts}

\usepackage{float}
\usepackage{caption}
\usepackage{subcaption}
\usepackage{appendix}
\usepackage{stmaryrd}
\usepackage{cleveref}
\newenvironment{sproof}{%
  \proof}{\endproof}
\renewcommand{\emph}{\textbf}

\newtheorem{lemma}{Lemma}
\newtheorem{theorem}{Theorem}
\theoremstyle{definition}
\newtheorem{definition}{Definition}
\title{Multi-Radar Tracking Optimization for Collaborative Combat}
\author[1]{Nouredine Nour}
\author[1]{Reda Belhaj-Soullami}
\author[2]{Cédric L.R. Buron}
\author[3]{Alain Peres}
\author[3]{Frédéric Barbaresco}

\affil[1]{
NukkAI, Paris, France (www.nukk.ai),\\ email: \{nnour, rbelhaj-soullami\}@nukk.ai}
\affil[2]{	Thales Research \& Technology, 1 av. Augustin Fresnel, Palaiseau France\\ email: cedric.buron@thalesgroup.com}
\affil[3]{Thales Land \& Air Systems, 3 Avenue Charles Lindbergh, 94150 Rungis, France  \\
        email: \{alain.peres, frederic.barbaresco\}@thalesgroup.com
}

\begin{document}

\maketitle

\begin{abstract}
\sloppy
Smart Grids of collaborative netted radars accelerate kill chains through more efficient cross-cueing over centralized command and control. In this paper, we propose two novel reward-based learning approaches to decentralized netted radar coordination based on black-box optimization and Reinforcement Learning (RL). To make the RL approach tractable, we use a simplification of the problem that we proved to be equivalent to the initial formulation. We apply these techniques on a simulation where radars can follow multiple targets at the same time and show they can learn implicit cooperation by comparing them to a greedy baseline. %Our results show that the RL solution was able to learn successfully even if it lacks generalisation. Our evolution approach even outperform a greedy heuristic by learning collaborative behaviors.

\end{abstract}

\section{Introduction}
Despite great interest in recent research, in particular in China \cite{dai2020sensor,shi2020joint} micromanagement of sensors by centralized command and control drives possible inefficiencies and risk into operations. Tactical decision making and execution by headquarters usually fail to achieve the speed necessary to meet rapid changes. %Smart Grid of collaborative netted sensors will accelerate kill chains through more efficient cross-cueing and will help C2 capabilities to increase the speed and accuracy of actions, especially against saturating and hyper-maneuvering targets. % Combining proactive multifunction radars with C2 will help empower collaborative teaming operations that diminish the vulnerability. 
Collaborative radars with C2 must provide decision superiority despite the attempts of an adversary to disrupt OODA cycles at all level of operations. Artificial intelligence can make a contribution for the purposes of coordinated conduct of the action, by improving the response time to threats and optimizing the allocation and the distribution of tasks within elementary smart radars.

%The coordination of assets is a key factor for accelerating the tempo of a maneuver in military operations in all environments.  
%Artificial intelligence can make a contribution for the purposes of coordinated conduct of the action, by improving the response time to threats and optimizing the allocation and the distribution of tasks within elementary smart radars. %Selecting the best compromise between centralized and decentralized sensors resources management will be a key challenge for collaborative combat.
% with a baseline heuristic solution, → pas sûr que la baseline soit une contribution à mettre en avant.

In order to address this problem, Thales and the private research lab NukkAI have been collaborating to introduce novel approaches for netted radars. Thales provided the simulation modeling the multi-radar target allocation problem and NukkAI proposed two novel reward-based learning approaches for the problem.

In this paper, we present these two approaches: Evolutionary Single-Target Ordering (ESTO), which is based on evolution strategies and an RL approach based on Actor-Critic methods. To make the RL method tractable in practice, we introduce a simplification of the problem that we prove to be equivalent to solving the initial formulation. We evaluate our solutions on diverse scenarios of the aforementioned simulation. By comparing them to a greedy baseline, we show that our algorithms can learn implicit collaboration. 
The paper is organized as follows: \cref{sota} introduces the related works. The problem is formalized in \cref{radar-sec}. In \cref{task-allocation}, we describe the proposed approaches. \Cref{evaluation} presents the results of our simulations. \cref{conclusion} concludes the paper.

\section{Related works}
\label{sota}
Decentralized target allocation in a radar network has gained a lot of interest recently \cite{geng2020evolution}. For this problem, resolution through non-cooperative game formalism \cite{shi2018non} reaches good performance, but only considers mono-target allocation. Bundle auctions algorithm \cite{jiang2019research} overcomes this limitation; still, none of these approaches are able to model the improvement provided by multiple radars tracking the same target. Another suitable method is reward-based machine learning, that can either take the form of evolutionary computation \cite{es_scalable_alt_to_rl} or reinforcement learning (RL). Recent successes in multi-agent RL were obtained by adapting mono-agent deep RL methods to the multi-agent case, most of them based on policy gradient approaches \cite{DBLP:conf/nips/SrinivasanLZPTM18} with a centralized learning and decentralized execution framework \cite{DBLP:journals/aamas/Hernandez-LealK19,ma_actor_critic,emergent}. In the policy gradient algorithm family, actor-critic methods \cite{a2c_paper} relying on a centralized critic have empirically proven to be effective in the cooperative multi-agent partially observable case \cite{ma_actor_critic}. However, the size of the action space requires to adapt these approaches for our problem. %They have been proven effective to tackle decentralized collaborative multi-agent problems \cite{DBLP:journals/aamas/PanaitL05}. 

\section{Problem statement}
\label{radar-sec}
In this paper, we consider that each radar can track a set of targets that move in a 2D space (for sake of simplicity, elevation is ignored). The targets are tracked with an uncertainty ellipse, computed using a linear Kalman model. We assume that the radars have a constant Signal-to-Noise Ratio (SNR) on target (adaptive waveform to preserve constant SNR on target), can communicate without limitations, have a limited field of view (FOV) and a budget representing the energy they can spend for tracking capabilities. Their Search mode are not simulated but taken into accoun the the constrained time budget for active track modes.

Let $n$ be the number of radars and $m$ the number of targets. In our model, an action of a radar is the choice of the set of targets to track. If multiple radars track the same target, the \emph{uncertainty area} is the superposition of the uncertainty ellipses. We define a utility function $\mathcal{U}$ measuring the global performance of the system. Let $l_i^j$ be the elementary radar $i$ budget needed to track target $j$, $L_i$ the budget of radar $i$, $\mathcal{E}_i^j$ the uncertainty ellipse on target $j$ for radar $i$ and $S(\mathcal{E}^i_j)$ the area of $\mathcal{E}^i_j$ (improvement is possible by considering covariance matrix of trackers). The problem can be expressed as a constraint optimization problem:
\begin{equation}
    \textrm{maximize }\mathcal{U} =\frac{1}{m} \sum\limits_{j=1}^m \exp\left[-S\left(\bigcap\limits_{i=1}^n \mathcal{E}_i^j\right)\right] \\\textrm{ such that: } \forall i\leq n, \sum_{j=1 }^{m}l_i^j\leq L_i
\end{equation}
%This function can be decomposed as a sum on utilities for each target, rewards a zero utility for untracked targets and assigns equal importance to all the targets.
\section{Task allocation in a radar network}
\label{task-allocation}

%To solve this problem we explored approaches from both black-box optimization and reinforcement learning. In this section, we present both approaches in detail, explaining how it can be used to solve the target allocation problem and provide advantages and drawbacks for each method.

\subsection{Evolutionary Single-Target Ordering (ESTO)}
ESTO is a centralized training with decentralized execution black-box optimization method. Based on contextual elements, agents define a preference score for each target. They then choose the targets to track greedily based on this score, until their budget is met. The preference score is computed by a parametrized function optimized to maximize the utility using the Covariance Matrix Adaptation Evolution Strategy (CMA-ES) \cite{hansen2001completely}. CMA-ES optimizes the parameters by sampling them from a normal distribution and updating both the mean and the covariance matrix based on the value of the utility function. We modelized ESTO's preference score function with a linear model with 9 input features including information on the target, the radar, and the position of the other radars. We also propose a variant of ESTO, called ESTO-M which takes into account 2 additional features based on inter-radar communication: the estimated target utility and the evolution of the estimated utility from the previous step.

\subsection{The Reinforcement Learning approach}
\paragraph{Dec-POMDP formulation.}
\sloppy
Collaborative multi-agent reinforcement learning typically relies on the formalization of the problem as a Dec-POMDP \cite{pomdp_book}. It is defined as a tuple $\left<D, S, A, P, R, \Omega, O\right>$, where $D$ is the set of agents ($|D| = n$), $S$ is the state space, $\mathcal{A} = \mathcal{A}_1\times\dots\times \mathcal{A}_n$ is the joint action set, $P$ is the transition function,
$R$ is the reward function, $\Omega$ is the set of observations and $O$ is the observation function. A state can be described as a tuple containing the true position and velocity of each target, the position of the other radars, and the Kalman filter parameters for each radar-target pair. The transition dynamics include the update of the Kalman filters. The problem is not fully observable: the radars approximate the position and speed of the targets and can't access to the Kalman filters of the others. Radars rely on a $m\cdot n_f$ real-valued vector, with $n_f$ the number of features: estimated position, speed, etc.

Our approach is based on centralized learning and decentralized execution. Although this approach may lead to stationarity issues, it is widely used in practice and yields good results in multi-agent RL \cite{DBLP:conf/nips/FoersterAFW16,ma_actor_critic}. When applying policy, the probability of tracking targets beyond FOV is set to 0, the probabilities of remaining targets are updated accordingly. Agents share the same network but its input values  depend on the radar and the targets.

In this setting, the size of the action space corresponds to all the possible allocations of sets of targets to each agent $\forall i\leq n,\mathcal{A}_i = \mathcal{P}(\llbracket 1, m\rrbracket)$ and $|\mathcal{A}_i| = 2^{m}$, \textsl{i.e.} the powerset of all targets. To tackle this issue, we propose a new formalization of the problem where the radars choose the target sequentially, and prove the equivalence between the solutions of the two formalisms.

%We split the action of choosing a set or targets into a sequence of states and actions where each actions corresponds to the choice of one new target (or none, action $\dagger$), and each state is enriched with the set of already chosen targets.
\sloppy

\begin{definition}[Sequential choice Dec-POMDP]
\label{def_new_mdp}

Let $M = \left<D, S, A, P, R, \Omega, O\right>$ be a Dec-POMDP for our problem. Let $M' = \left<D, S',A',P',R', \Omega, O'\right>$ with $S' = S\cup \left\{s\otimes \varepsilon| s\in S, \varepsilon \in \mathcal{P}(\llbracket 1, m \rrbracket\cup \dagger)^n\right\}$ where $\varepsilon$ is the set of targets tracked by each agent and the symbol $\dagger$ means that its allocation is finished; $s\otimes \varepsilon$ is a notation for the new (couple) state in $S'$. In a state  $ s\otimes \varepsilon\in S'$, the set of allowed actions of agent $i$ is $(\llbracket 1, m \rrbracket \cup \{\dagger \})\setminus \varepsilon_i$. The observation, state-transition and reward functions are defined as: 

$$
\begin{aligned}
\forall a\in A',\forall(\varepsilon, \varepsilon') \in \left(\mathcal{P}(\llbracket 1, m \rrbracket)^n\right)^2, \forall (s, s')\in S^2, O'(s\otimes\varepsilon,s'\otimes\varepsilon) = O(s,s')\\
    P'(s\otimes \varepsilon, a, s'\otimes \varepsilon') = \begin{cases}
    P(s,\varepsilon,s')\textrm{ if } a = (\dagger,\dots,\dagger), \varepsilon' = \varnothing\\
    1 \textrm{ if } s = s', \varepsilon'_j = \varepsilon_j\cup\{a_j\}\forall j \leq n\\
    0 \textrm{ else}
    \end{cases}\\
    R'(s\otimes \varepsilon, a, s'\otimes \varepsilon') = \begin{cases}
    R(s,\varepsilon,s')\textrm{ if } a = (\dagger,\dots,\dagger), \varepsilon' = \varnothing\\
    0 \textrm{ else}
    \end{cases}
\end{aligned}$$
\end{definition}
This new Dec-POMDP can be solved much more easily than the initial one. We now look for a solution of the initial Dec-POMDP from the sequential choice Dec-POMDP. In the rest of the article, we denote by $V$ (resp. $V'$) the averaged state value function in $M$, (resp. $M'$): $V_ \pi(\rho) =\mathbb{E}_{a_t \sim \pi, s_0 \sim \rho}\left(\sum_{t=1}^T R(s_t, a_t, s'_t)\right)$. For space reasons, only sketches of the lemma proofs are provided.

\begin{definition}[Policy transposition]
\label{def_phi}
We define the policy transition function $\phi$ from the set of policies in $M'$ to the set of policies in $M$ as 
\[ \forall j\in\llbracket 1, n\rrbracket,  \phi_j(\pi')(\varepsilon|\omega) =
%\sum_{\{i_1, \ldots, i_p\} = \varepsilon}
\sum_{\{i_k \} = \varepsilon}
\pi_j'(i_1 \vert \omega) \pi_j'(i_2 \vert \omega \otimes \{i_1\}) \ldots \pi_j'(\dagger \vert s \otimes \{i_1, \ldots, i_p\})\]
\end{definition}

\begin{lemma}[Value equivalence]
\label{value_equal}
Let $\pi = \phi(\pi')$. Let $\rho$ be a probability distribution on $S$ and $\pi'$ a policy on $M'$. Then $V'_{\pi'}(\rho) = V_{\phi(\pi)}(\rho)$.
\end{lemma}

\begin{sproof}
By \cref{def_new_mdp}, the result holds iff $\forall (s,\varepsilon,s') \in S \times \mathcal P(\llbracket 1, m \rrbracket)^n \times S,$
\[
\rho(s)\pi(\varepsilon \vert \omega)P(s,\varepsilon, s') = %\sum_{\{ i_1,\ldots, i_p \} = \varepsilon}
\sum_{ \{i_k \} = \varepsilon}
\rho(s) P'(s,i_1,s\otimes i_1) \pi'(i_1 \vert \omega) \ldots P'(s \otimes \varepsilon, \dagger, s') \pi'(\dagger \vert \omega \otimes \varepsilon)\]
By using \cref{def_new_mdp} the equation simplifies exactly to the one of \cref{def_phi}.
\end{sproof}

\begin{lemma}[Surjectivity]
\label{phi_bij}
The mapping $\phi$ is surjective. Let $\pi$ be a policy on $M$.  Then $\pi = \phi(\pi')$ with $\pi'$ defined the following way ($\omega$ is omitted):
\begin{equation}\label{piprime}
\begin{aligned}
&\forall k\in\llbracket1,n\rrbracket, \forall j \in \llbracket 1, m \rrbracket \setminus \varepsilon_k, \pi_k'(\dagger \vert \varepsilon_k) = \frac{1}{N_{\varepsilon_k}} \pi_k(\varepsilon_k) \frac{\pi_k(A) \vert \varepsilon_k \vert !}{\vert A \vert \ldots(\vert A \vert - \vert \varepsilon_k \vert)}
\\&\textrm{and }
\pi_k'(j \vert \varepsilon_k) = \frac{1}{N_{\varepsilon_k}}\sum_{\substack{
A \subset \llbracket 1, m \rrbracket \\ \varepsilon \subset A  \\ j \in A 
}}
\textrm{ with }N_{\varepsilon_k}  = \sum_{\substack{
A \subset \llbracket 1, m \rrbracket \\ \varepsilon_k \subset A  
}} \frac{\pi_k(A) \vert \varepsilon_k \vert ! }{\vert A \vert \ldots (\vert A \vert- \vert \varepsilon_k \vert +1)}
\end{aligned}
\end{equation}
\end{lemma}

\begin{sproof}

First, we verify that $\sum_{\substack{j \in \llbracket 1, m \rrbracket \backslash \varepsilon_k}} \pi_k'(j \vert \varepsilon_k)= 1 - \pi_k'(\dagger\vert \varepsilon_k).$
Then we show that $\pi_k = \phi(\pi_k')$. Let $\varepsilon = (i_1, \ldots, i_p)$, and $(j_1, \ldots, j_p)$ an arbitrary permutation of $\varepsilon$. Let $\varepsilon_l = (j_1, \ldots, j_l)$. Then, for all $l \in \llbracket 0, p-1 \rrbracket, \pi_k'(j_{l+1} \vert \varepsilon_l) = \frac{N_{\varepsilon_{l+1}}}{N_{\varepsilon_{l}}(l+1)} $ with $\varepsilon_0 = \varnothing$.
The product then simplifies to $\pi_k'(j_1) \pi_k'(j_2 \vert j_1) \ldots \pi_k'(\dagger \vert \{ j_1, \ldots, j_p\}) = \dfrac{\pi_k(\varepsilon_k)}{p!}$.
Summing among all $p!$ permutations of $\llbracket 1, p\rrbracket$, we verify that $\pi_k = \phi(\pi_k')$.
\end{sproof}

\begin{theorem}
Let $\pi_*'$ be an optimal policy in $M'$, then $\pi_* = \phi(\pi_*')$ is an optimal policy in $M$.
\end{theorem}

\begin{proof}
% Let $\pi$ be a policy on $M$. By \cref{phi_bij} we can define $\pi' = \phi^{-1}(\pi)$. By \cref{value_equal} $V'_{\pi'}(\rho) = V_\pi(\rho)$. As $\pi_*'$ is optimal we have $ V'_{\pi_{*}^{'}}
% (\rho) \geq V'_{\pi'}(\rho)$. Then, $V_\pi(\rho) \leq V_{\pi^*}(\rho)$. 
This follows directly from the surjectivity and value equivalence lemmas.
\end{proof}
\paragraph{Actor-Critic methods.}
To find a policy that maximizes the expected average reward, we used Proximal Policy Optimization (PPO) \cite{ppo_paper}, a variant of the actor-critic algorithm. Although the algorithm is only proved for MDPs, the use of a centralized critic has proven to be efficient in simple partially observable multi-agent settings \cite{ma_actor_critic}. PPO relies on an actor that plays episodes according to a parametrized policy and a critic that learns the state value function. After each batch of played episodes, the parameters of the two networks are updated according to the surrogate loss as in \cite{ppo_paper}.

\paragraph{Neural network architecture.}
The critic neural network architecture is a standard multi-layer perceptron. Regarding the actor, the first layer consists of $n_f$ neurons : an input tensor of size $(m, n_f)$ is passed to the network instead of a first layer of $n_{f}\cdot m$ neurons. The network consists of a \emph{feature extractor} of two layers reducing the number of features from 23 to 6 and a \emph{feature aggregator} consisting of two linear models $T$ and $O$, that represent respectively the contribution of the target itself, and the interest of the other targets. Intuitively, training at individual target level allows better feature extraction and generalization than a dense, fully connected architecture. Moreover, it allows to ensure full symmetry of the weights. However, this comes at the cost of expressiveness, as we use a special form of architecture for our actor.
Let $f$ be the extracted feature matrix : $f_i$ is the extracted features for target i. 
We compute the score $w_i$ of target $i$ as 
$w_i = T(f_i) + \frac{1}{m-1} \sum_{j=1, i \neq j }^{m} O(f_j)$. The process is converted to a probability using a softmax activation function and can be represented as:
\[3 @ 23 \textrm{ feat.} \rightarrow \textrm{feature extractor} \rightarrow 3 @ 6 \textrm{ feat.} \rightarrow \textrm{feature aggregator} \rightarrow 3 \textrm{ scores}\]
\section{Evaluation}
\label{evaluation}
\begin{figure}[!htbp]
     \centering
     \begin{subfigure}[t]{0.19\textwidth}
         \centering
         \includegraphics[width=\textwidth]{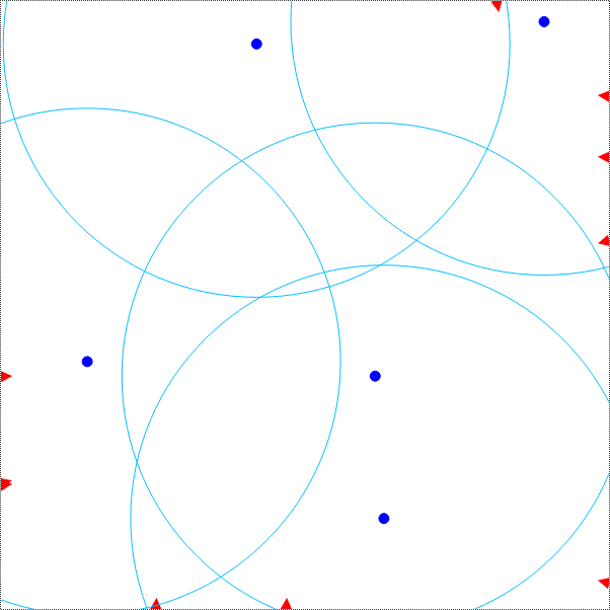}
         \caption{Non saturated well positioned radars}
     \end{subfigure}
     \begin{subfigure}[t]{0.19\textwidth}
         \centering
         \includegraphics[width=\textwidth]{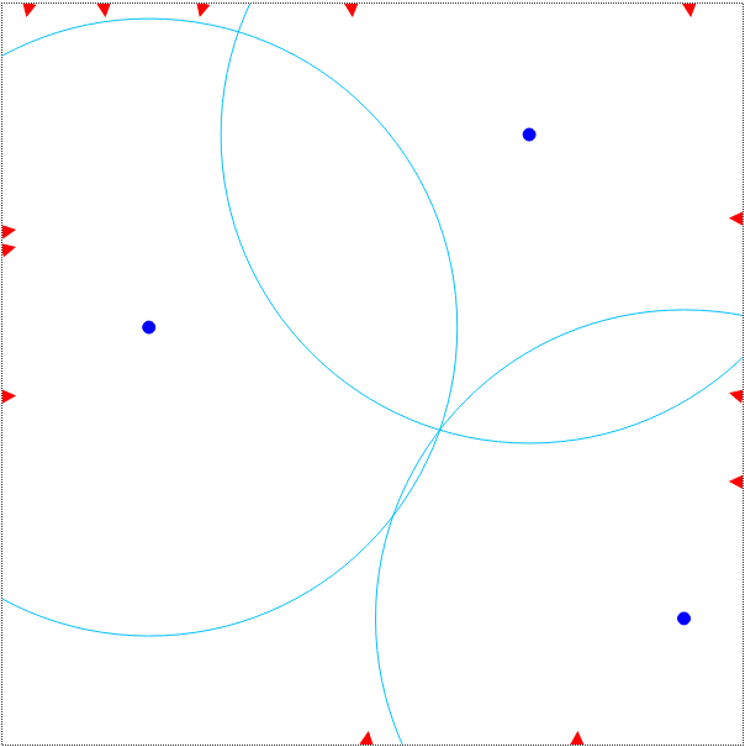}
         \caption{Few saturated well positioned radars}
     \end{subfigure}
     \begin{subfigure}[t]{0.19\textwidth}
         \centering
         \includegraphics[width=\textwidth]{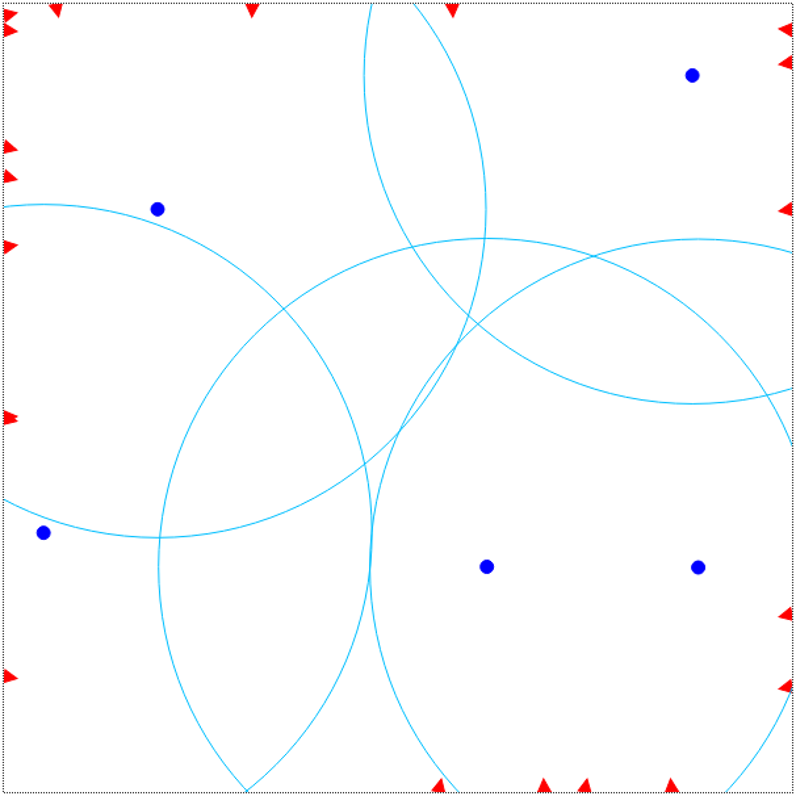}
          \caption{Several saturated well positioned radars}
     \end{subfigure}
          \begin{subfigure}[t]{0.19\textwidth}
         \centering
         \includegraphics[width=\textwidth]{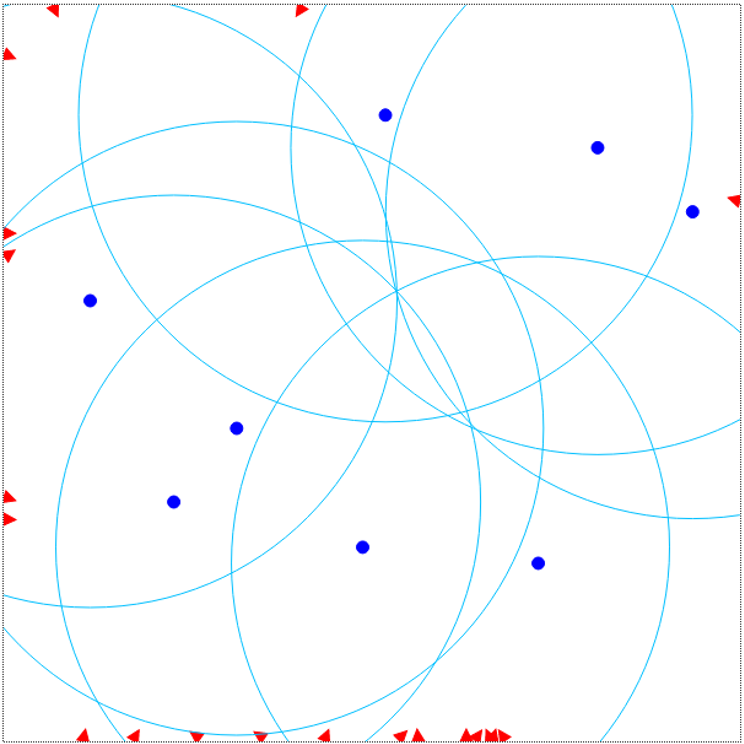}
          \caption{Many saturated well positioned radars}
     \end{subfigure}
    \begin{subfigure}[t]{0.19\textwidth}
         \centering
         \includegraphics[width=\textwidth]{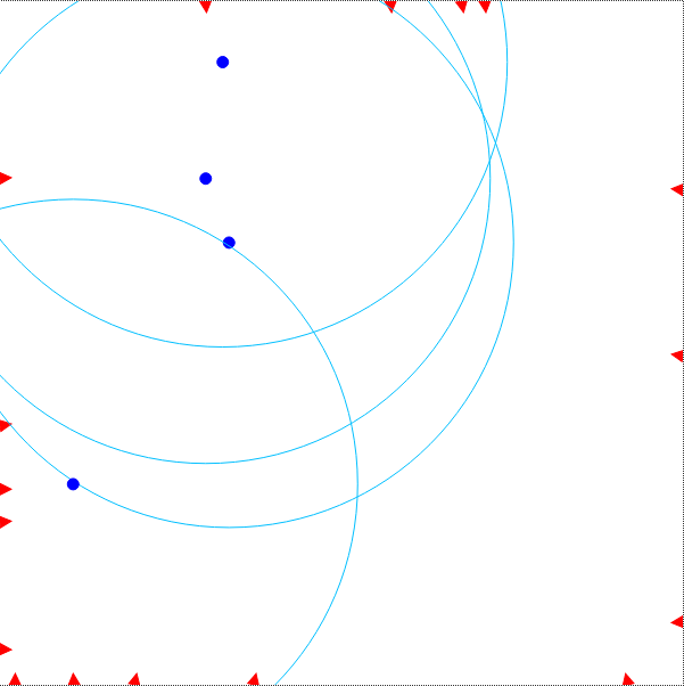}
          \caption{Saturated ill-positioned radars}
     \end{subfigure}
        \caption{The 5 validation scenarios (radars \& FOV in blue, targets in red)}
        \label{fig:three graphs}
\end{figure}

\begin{figure}[!htb]
\hfill
     \centering
     \begin{subfigure}[t]{0.38\textwidth}
         \centering
         \includegraphics[width=\textwidth]{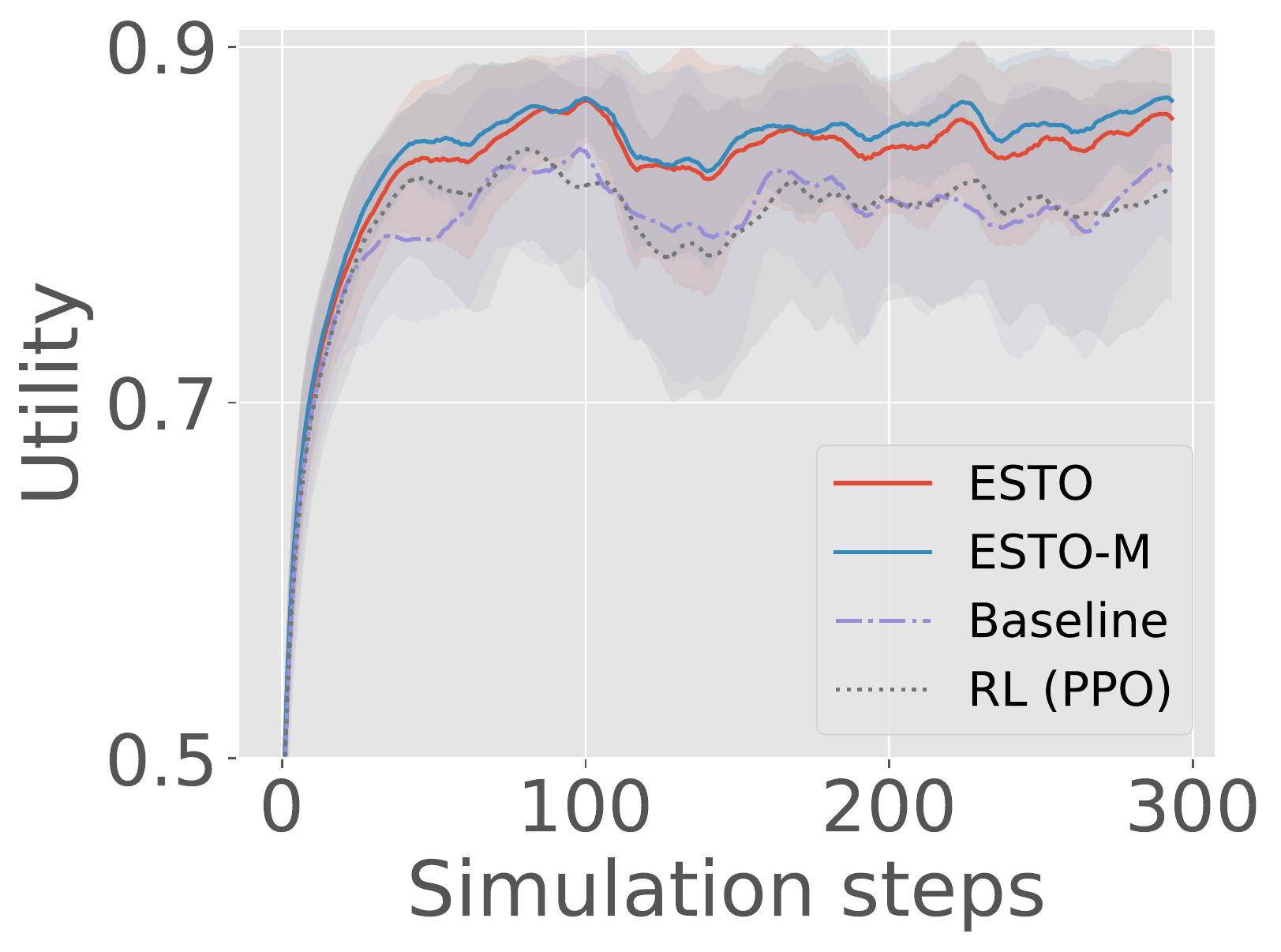}
     \subcaption{Scenario a}
     \end{subfigure}
     \begin{subfigure}[t]{0.3\textwidth}
         \centering
         \includegraphics[width=\textwidth]{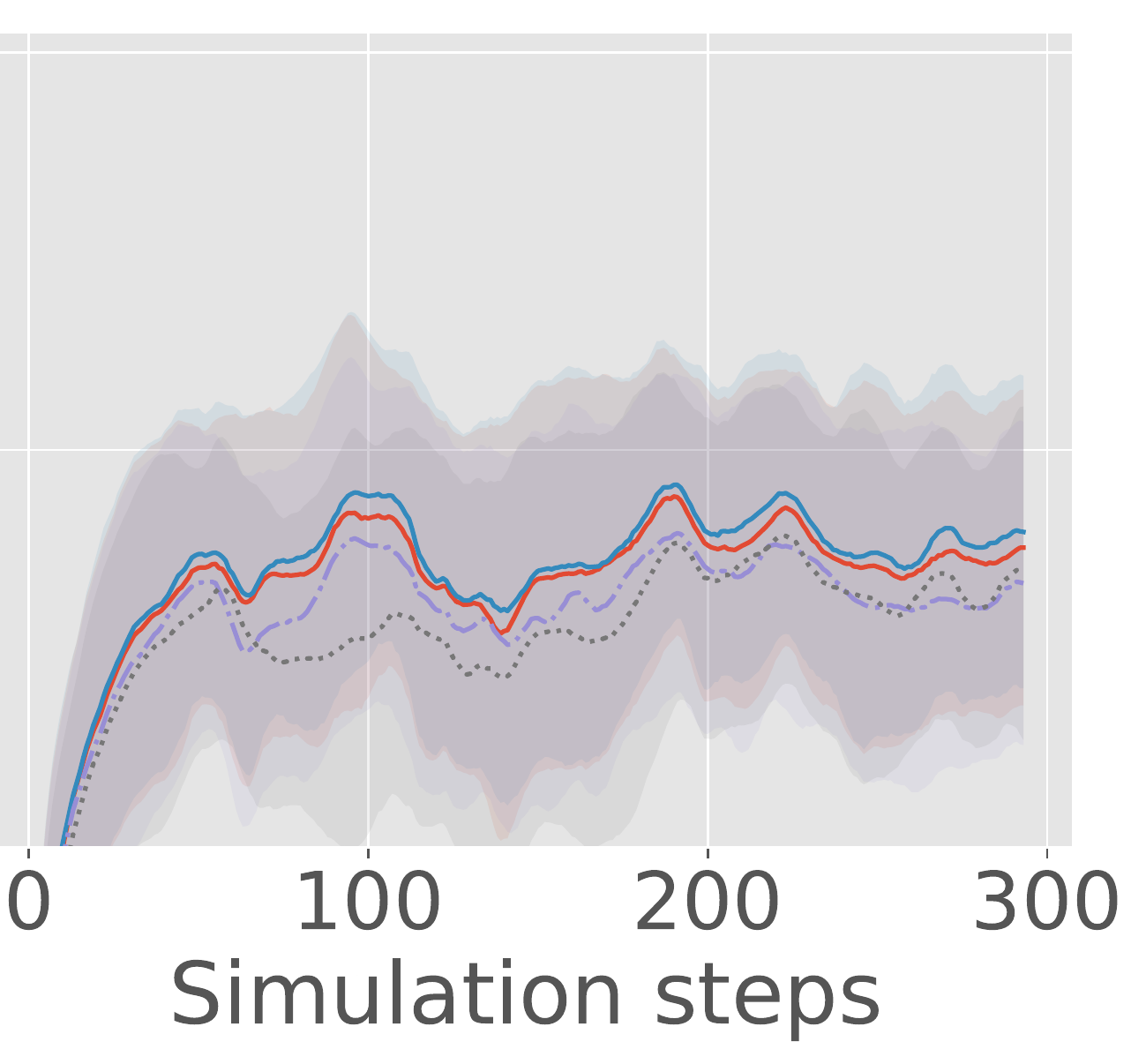}
     \subcaption{Scenario b}
     \end{subfigure}
     \begin{subfigure}[t]{0.3\textwidth}
         \centering
         \includegraphics[width=\textwidth]{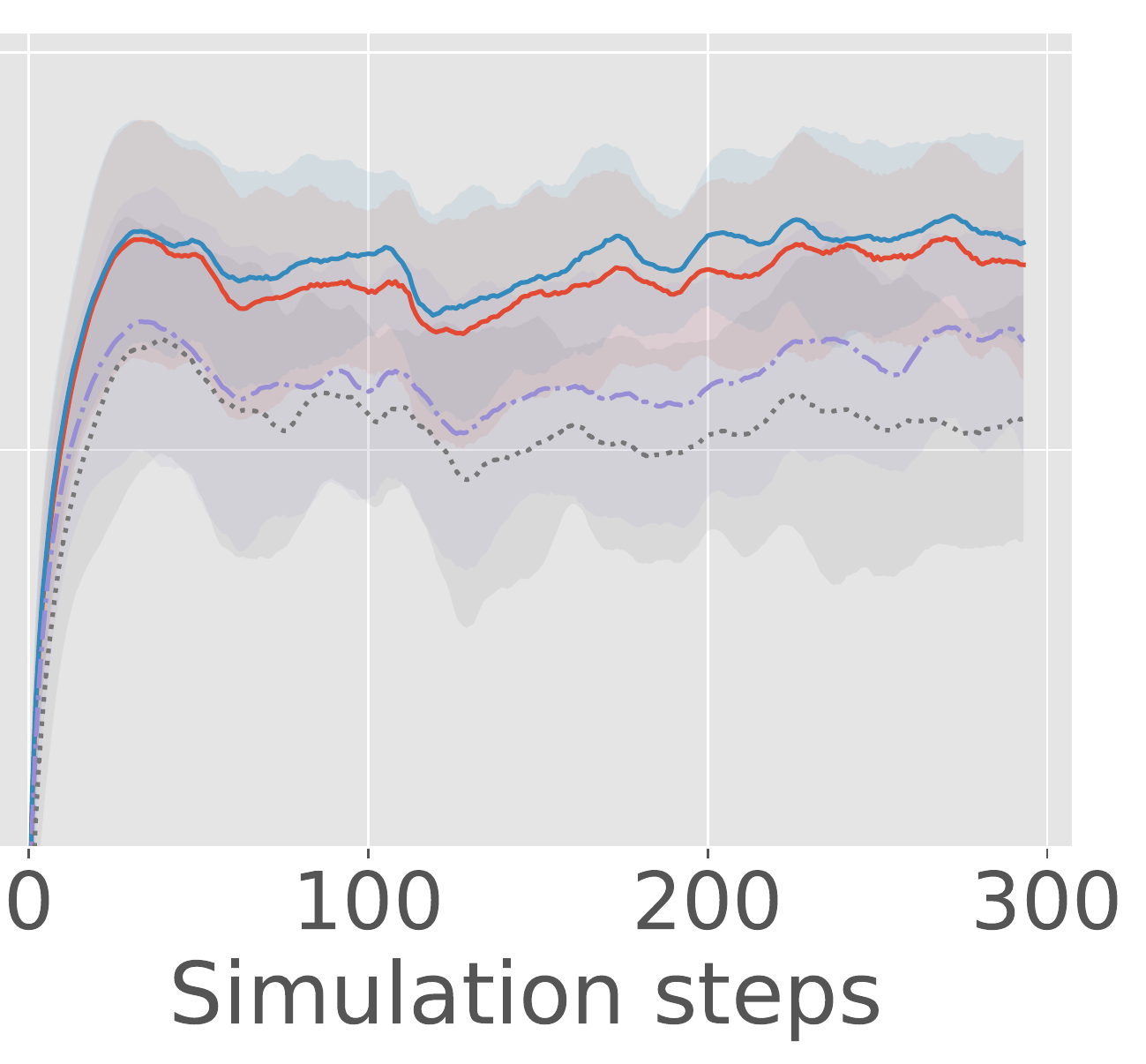}
     \subcaption{Scenario c}
     \end{subfigure}
     \begin{subfigure}[t]{0.38\textwidth}
         \centering
         \includegraphics[width=\textwidth]{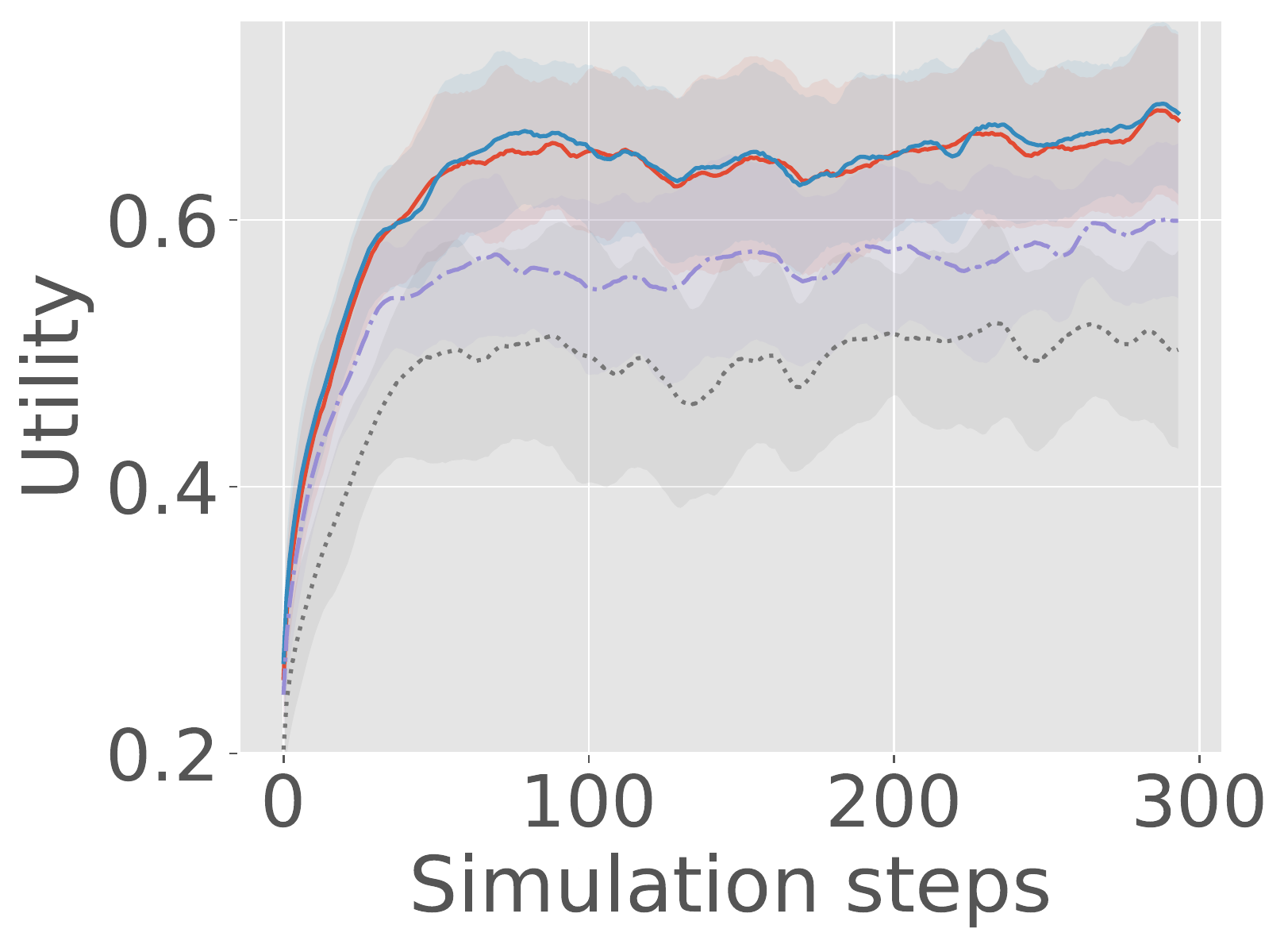}
     \subcaption{Scenario d}
     \end{subfigure}
     \begin{subfigure}[t]{0.3\textwidth}
         \centering
         \includegraphics[width=\textwidth]{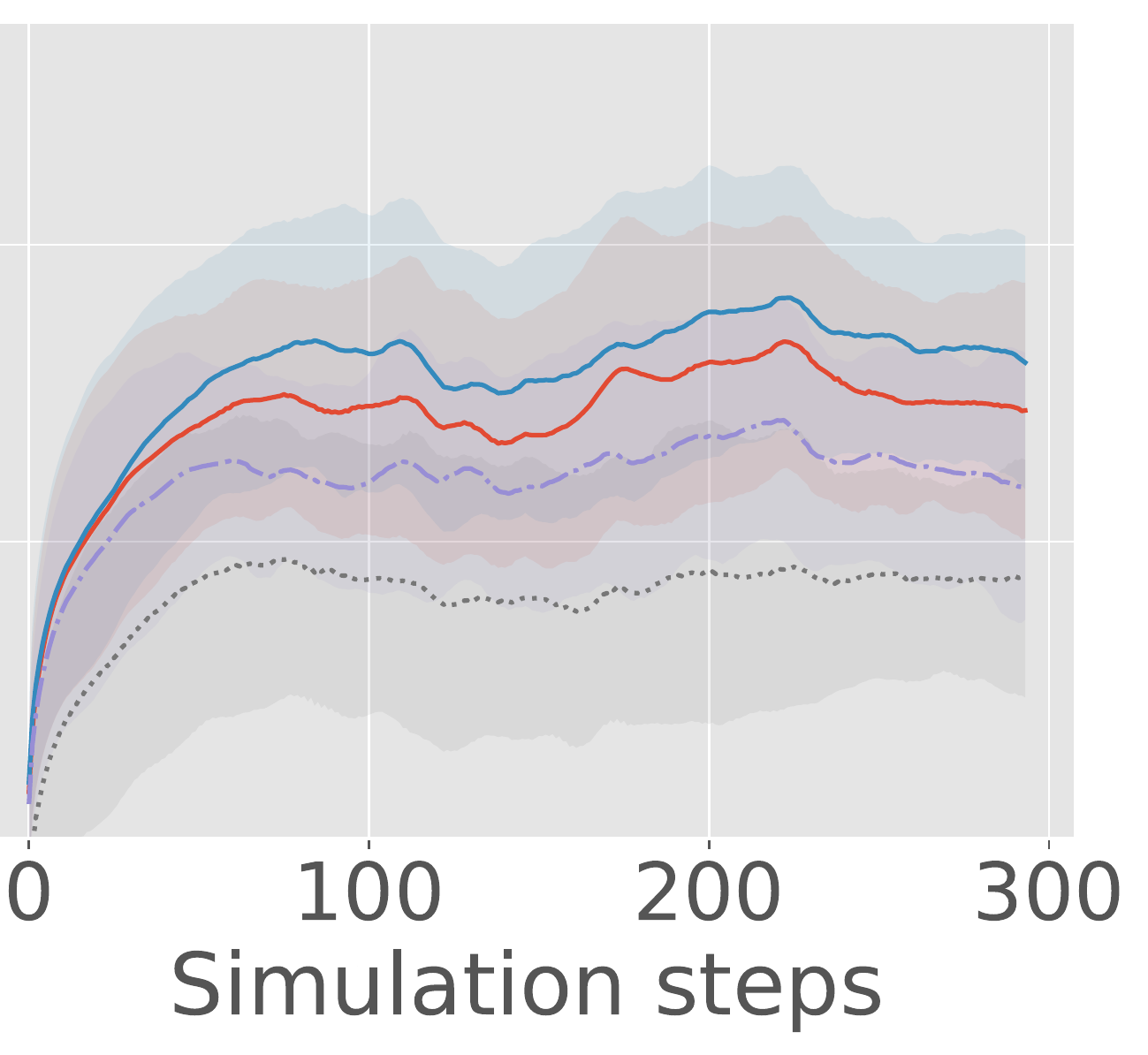}
     \subcaption{Scenario e}
     \end{subfigure}
     \begin{subfigure}[t]{0.3\textwidth}
         \centering
         \includegraphics[width=\textwidth]{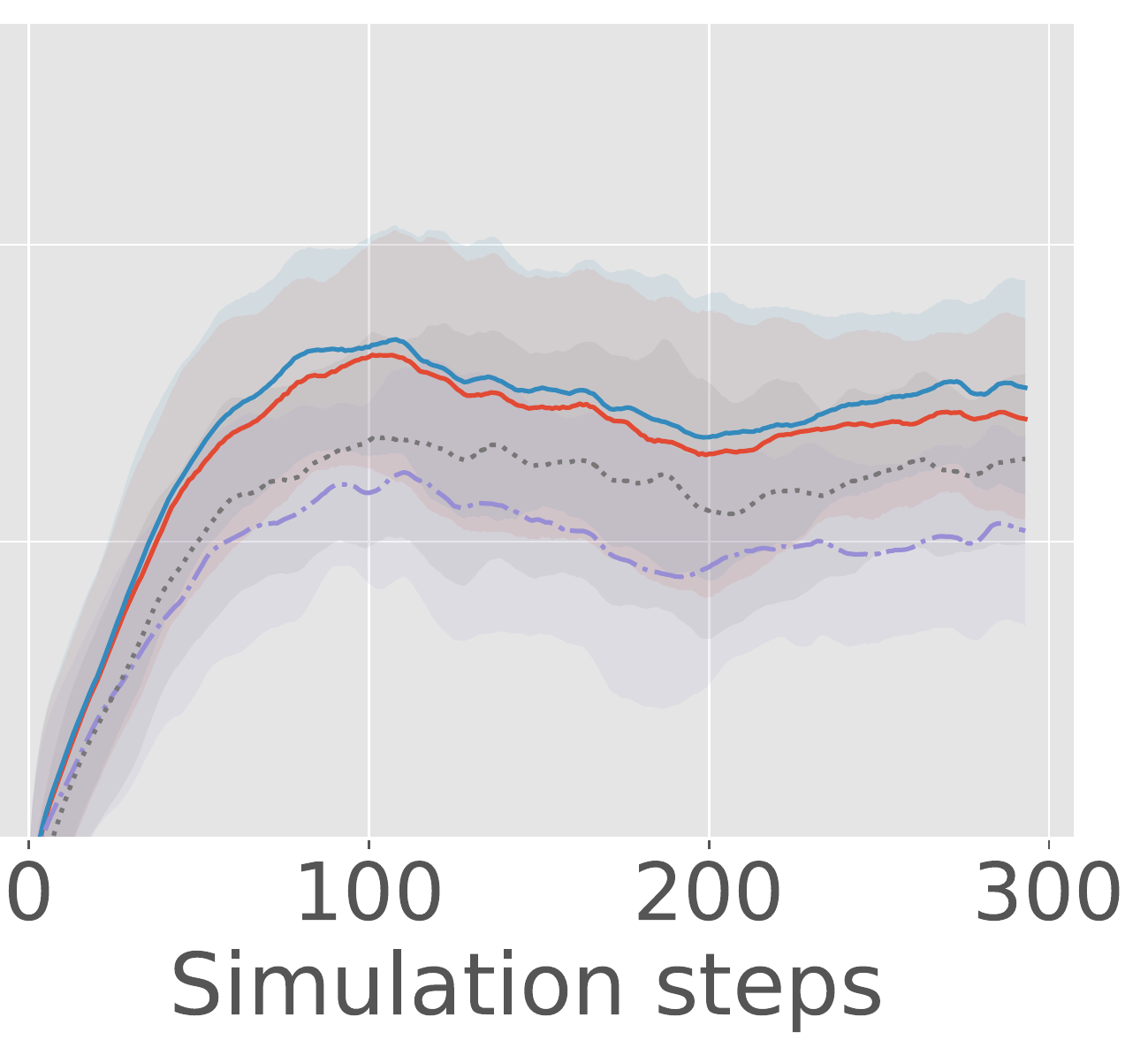}
     \subcaption{RL training scenario}
     \label{fig-RL}
     \end{subfigure}
        \caption{Utility of the radars over 300 steps with 16 random seeds}
        \label{fig-utility}
\end{figure}

\begin{figure}[!htb]
     \centering
     \begin{subfigure}[b]{0.38\textwidth}
         \centering
         \includegraphics[width=\textwidth]{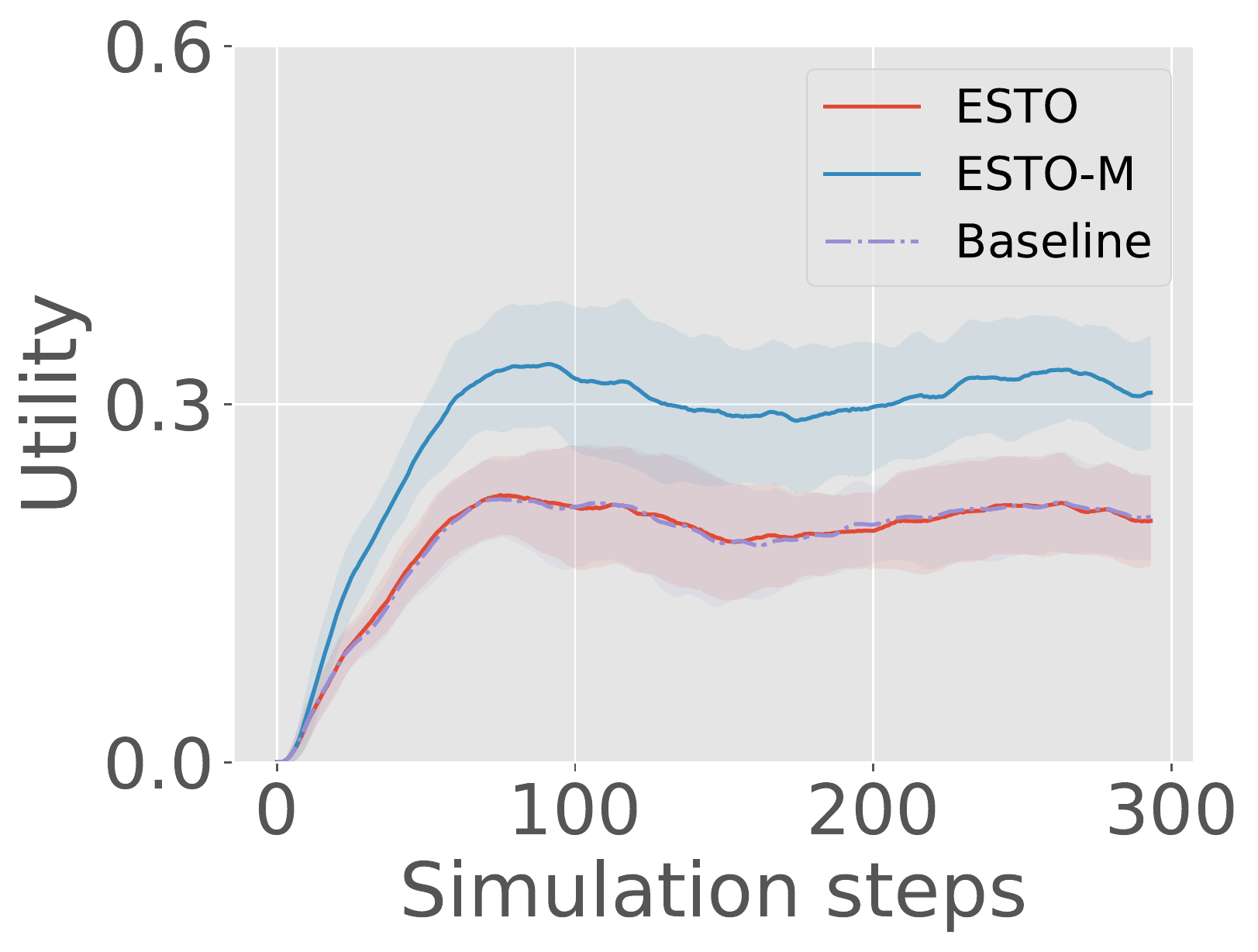}
          \subcaption{2 stacked radars}
     \end{subfigure}
     \begin{subfigure}[b]{0.3\textwidth}
         \centering
         \includegraphics[width=\textwidth]{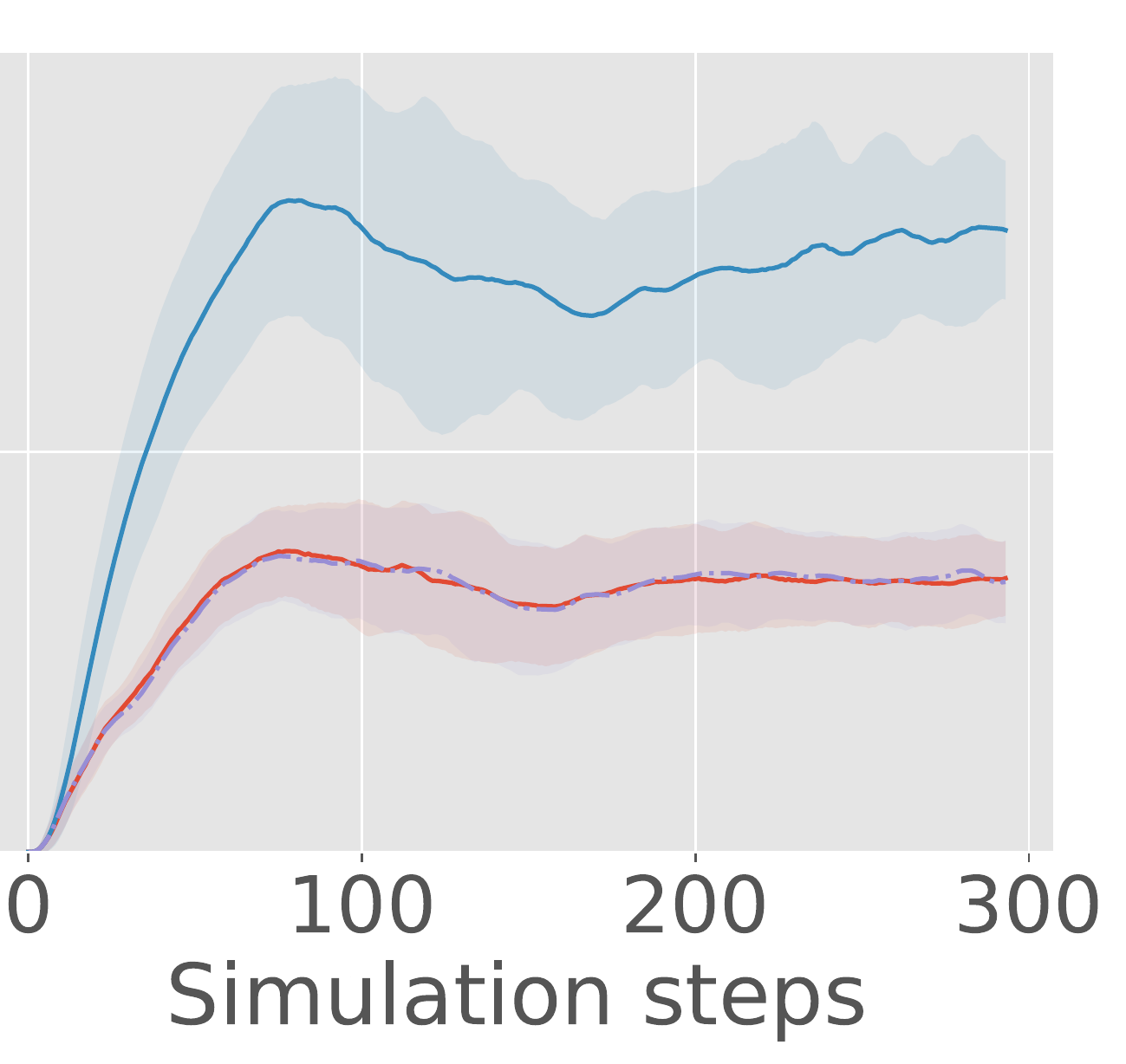}
         \subcaption{8 stacked radars}
     \end{subfigure}
     \hfill
        \caption{Performance for stacked radars scenarios}
        \label{fig:stacked}
\end{figure}
The evaluation is performed on a multi-agent simulator built with the mesa framework~\cite{masad2015mesa}. It consists of an environment of fixed size without obstacles with two kinds of agents: \emph{the radars} are implemented according to the model presented in \cref{radar-sec}. In order to simplify the model, we make the following assumptions: the radars rotate at the rate of 1 round/step; \emph{the targets} have a fixed speed and direction. They can turn by up to 15\textdegree{} and are replaced when they leave the simulation space. The simulator uses a random scheduler \textsl{i.e.} the agents act in a random order. The information they use may therefore be outdated, which allows to check the system resilience when the agents don't have up-to-date information. The ESTO approach is optimized on a scenario with 3 fixed radars and 20 targets with random trajectories over 10 runs to enhance generalization. The RL agent trains on the same settings, but over one run only due to time constraints. Our approaches are compared to a simple ``baseline'' approach (the radars greedily select the closest targets) on the 5 scenarios provided in \cref{fig:three graphs}, representing interesting configurations of the radar network.

As \cref{fig-utility} shows, both ESTO and ESTO-M significantly outperform the baseline on all scenarios. The performance gain seems to be correlated with the overlap of the agents field of view (FOV). When the FOV overlap is minimal, there is less need for cooperation between agents and the baseline is close to being optimal. Conversely, when the overlap is maximal, cooperation is needed to achieve a good performance. Indeed, when radars are stacked (\cref{fig:stacked}), ESTO-M performs significantly better than the baseline, even more so as more radars are stacked unlike ESTO. This indicates that the distance to the closest radar feature plays an important role in ESTO's collaboration. This is confirmed by the fact that we do not observe a significant difference between ESTO and ESTO-M in scenarios (a) to (e) when ESTO can use the feature. The RL approach relies on the reformulation of the problem (\cref{def_new_mdp}). It outperforms the baseline on the training scenario but seems to have poor generalization.

\section{Conclusion}
\label{conclusion}
 In this paper, we presented two novel reward-based learning algorithms for multi-radar multi-target allocation based on centralized learning and decentralized execution. The first one, ESTO, relies on CMA-ES to optimize a linear preference function that is used to order targets for a greedy selection. The second is an actor-critic based reinforcement learning approach relying on a specific Dec-POMDP formalization. While ESTO significantly outperforms our greedy baseline by learning cooperative behaviors, the RL approach still lacks generality to do so systematically. Training it longer on more diverse scenarios (target trajectory, radar positions,  number of steps) may help to prevent overfitting.
Moreover, future improvements may include: 
The development of a neural version of ESTO that would rely on a large scale CMA-ES implementation \cite{varelas2018comparative} to handle the increase in the size of the parameter space. Another improvement would be the development of a more realistic radar simulation taking into account \textsl{e.g.} changes in SNR and rotation speed, and include obstacles and targets of different classes and priorities. More importantly, other than simply tuning our models for better numerical performance, we would like to interface them with symbolic AI methods allowing them to leverage expert domain knowledge and opening the way for explainable AI (XAI) developments.

%The presented methods are not well suited to make explainable decisions. Augmenting them with a symbolic module that could leverage expert domain knowledge and provide easier path towards explainability is a developement that NukkAI and Thales would like to adress in a future work.

 \bibliographystyle{unsrt}
	\bibliography{biblio.bib}

\end{document}